\documentclass[12pt]{article}
 
 \usepackage[page,header]{appendix}
\usepackage{titletoc}

\usepackage{todonotes}
\usepackage{graphicx}
\usepackage[numbers]{natbib}
\usepackage{etex}
\usepackage{enumitem}
\setlist{parsep = -0em, itemsep = 0.25em}
\usepackage[utf8]{inputenc}
\usepackage[dvips,letterpaper,margin=1in]{geometry}
\usepackage{amsmath}
\usepackage{amssymb}
\usepackage{amsthm}
\usepackage{bm}
\usepackage{colonequals}
\usepackage{comment}
\usepackage{graphicx}
\usepackage{subcaption}
\usepackage{xcolor}
\usepackage{tikz} % replaced by YT % \usepackage{pgfplots}
\usetikzlibrary{shadows} % YT: for adding shadow in TikZ plots
\usepackage{authblk}
\usepackage[utf8]{inputenc} % allow utf-8 input
\usepackage[T1]{fontenc}    % use 8-bit T1 fonts
\usepackage{url}            % simple URL typesetting
\usepackage{booktabs}       % professional-quality tables 
\usepackage{amsfonts}       % blackboard math symbols
\usepackage{nicefrac}       % compact symbols for 1/2, etc.
\usepackage{microtype}      % microtypography
\usepackage{hyperref}
\hypersetup{colorlinks = true,
            linkcolor = blue,
            urlcolor  = teal,
            citecolor = magenta,
            anchorcolor = blue}

\usepackage{amsmath,amsthm,amssymb}
\usepackage{algorithm,algpseudocode}
\usepackage{wrapfig}
\usepackage{esint}

\usepackage[margin=1in]{geometry} 
\usepackage{amsmath,amsthm,amssymb,bbm,mathtools}
\usepackage{xcolor}
\usepackage{ytableau}
\usepackage{caption}
\usepackage{subcaption}

\newcommand{\N}{\mathbb{N}}

\newcommand{\A}{\mathcal{A}}

\newcommand{\one}{\mathbf{1}}

\newcommand{\btheta}{\boldsymbol{\theta}}

\renewcommand{\S}{{S^1}}

\newcommand{\ind}{\mathbbm{1}}

\newcommand*\conj[1]{\overline{#1}}

\newtheorem{theorem}{Theorem}[section]
\newtheorem{lemma}[theorem]{Lemma}

\newtheorem{definition}[theorem]{Definition}

\begin{document}
% \obeylines
 
\title{Towards Antisymmetric Neural Ansatz Separation}
\date{}
\author[a]{Aaron Zweig}
\author[a,b]{Joan Bruna}

\affil[a]{Courant Institute of Mathematical Sciences, New York
  University, New York}
\affil[b]{Center for Data Science, New York University}

\maketitle

\begin{abstract}
We study separations between two fundamental models (or \emph{Ansätze}) of antisymmetric functions, that is, functions $f$ of the form $f(x_{\sigma(1)}, \ldots, x_{\sigma(N)}) = \text{sign}(\sigma)f(x_1, \ldots, x_N)$, where $\sigma$ is any permutation. 
These arise in the context of quantum chemistry, and are the basic modeling tool for wavefunctions of Fermionic systems. 
Specifically, we consider two popular antisymmetric Ansätze: the Slater representation, which leverages the alternating structure of determinants, and the Jastrow ansatz, which augments Slater determinants with a product by an arbitrary symmetric function. We construct an antisymmetric function in $N$ dimensions that can be efficiently expressed in Jastrow form, yet provably cannot be approximated by Slater determinants unless there are exponentially (in $N^2$) many terms. This represents the first explicit quantitative separation between these two Ansätze.

\end{abstract}

\section{Introduction}

%[INTRO]
Neural networks have proven very successful in parametrizing non-linear approximation spaces in high-dimensions, thanks to the ability of neural architectures to leverage the physical structure and symmetries of the problem at hand, while preserving universal approximation. The successes cover many areas of engineering and computational science, from computer vision \citep{krizhevsky2017imagenet} to protein folding \citep{jumper2021highly}.

In each case, modifying the architecture (e.g. by adding layers, adjusting the activation function, etc.) has intricate effects in the approximation, statistical and optimization errors. An important aspect in this puzzle is to first understand the approximation abilities of a certain neural architecture against a class of target functions having certain assumed symmetry~\citep{lecun1995convolutional,cohen2018spherical}.  For instance, symmetric functions that are permutation-invariant, ie $f(x_{\sigma(1)}, \ldots, x_{\sigma(N)}) = f(x_1, \ldots x_N)$ for all $x_1, \ldots, x_N$ and all permutations $\sigma: \{1, N\} \to \{1, N\}$ can be universally approximated by several neural architectures, e.g DeepSets~\citep{zaheer2017deep} or Set Transformers~\citep{lee2019set}; their approximation properties \citep{zweig2022exponential} thus offer a first glimpse on their efficiency across different learning tasks. 

In this work, we focus on quantum chemistry applications, namely characterizing ground states of many-body quantum systems. These are driven by the fundamental Schröndinger equation, an eigenvalue problem of the form
$$H \Psi = \lambda \Psi~,$$
where $H$ is the Hamiltonian associated to a particle system defined over a product space $\Omega^{\otimes N}$, and $\Psi$ is the \emph{wavefunction}, a complex-valued function $\Psi: \Omega^{\otimes N} \to \mathbb{C}$ whose squared modulus $|\Psi(x_1, \ldots, x_N)|^2$ describes the probability of encountering the system in the state $(x_1, \ldots, x_N) \in \Omega^{\otimes N}$. 
A particularly important object is to compute the \emph{ground state}, associated with the smallest eigenvalue of $H$. 
On Fermionic systems, the wavefunction satisfies an additional property, derived from Pauli's exclusion principle: the wavefunction is \emph{antisymmetric}, meaning that 
\begin{equation}
    \Psi(x_{\sigma(1)}, \ldots, x_{\sigma(N)}) = \text{sign}(\sigma) \Psi(x_1, \ldots, x_N)~.
\end{equation}

The antisymmetric constraint is an uncommon one, and therefore demands unique architectures to enforce it.  The quintessential antisymmetric function is a \emph{Slater determinant}~\citep{szabo2012modern}, that we now briefly describe. Given functions $f_1, \dots, f_N: \Omega \to \mathbb{C}$, they define a rank-one tensor mapping $f_1 \otimes \dots \otimes f_N : \Omega^{\otimes N} \to \mathbb{C}$ by $(f_1 \otimes \dots \otimes f_N)(x_1, \ldots, x_N):= \prod_{j\leq N} f_j(x_j)$. The Slater determinant is then the orthogonal projection of a tensor rank one function into antisymmetric space.  In other words, the rank one tensor $f_1 \otimes \dots \otimes f_N$ is projected to $\mathcal{A}(f_1 \otimes \dots \otimes f_N) :=\frac{1}{N!}\sum_{\sigma \in S_N} (-1)^\sigma f_{\sigma(1)} \otimes \dots \otimes f_{\sigma(N)}$. In coordinates, this expression becomes
\begin{align*}
\mathcal{A}(f_1 \otimes \dots \otimes f_N)(x) = \frac{1}{N!} \text{det}\begin{bmatrix} f_1(x_1) & \ldots & f_1(x_N) \\
f_2(x_1) & \ldots & f_2(x_N) \\
& \dots & \\
f_N(x_1) & \ldots & f_N(x_N) 
\end{bmatrix}~,
\end{align*}
which shows that is antisymmetric following the alternating property of the determinant. 

The Slater Ansatz is then simply a linear combination of several Slater determinants, of the form $F(x)=\sum_{l\leq L} \mathcal{A}(f_1^l \otimes \dots \otimes f_N^l)$, similarly as a shallow (Euclidean) neural network formed as a linear combination of simple non-linear ridge functions. 
While this ansatz defines a universal approximation class for antisymmetric functions (as a direct consequence of Weierstrass universal approximation theorems for polynomials), the approximation rates will generally be cursed by the dimensionality of the input space, as is also the case when studying lower bounds for standard shallow neural networks \cite{maiorov1998approximation}. 

In the case of particles in $\Omega = \mathbb{R}$ or $\mathbb{C}$, it is classical that all antisymmetric functions can be written as a product of a symmetric function with the Vandermonde (see Section~\ref{sec:prelim}).  This setting is generally considered much easier than settings with higher-dimensional particles, as this Vandermonde factorization no longer applies above one dimension, though there are still ansätze that mimic this formulation~\citep{han2019solving}.

A more powerful variant is the Jastrow Ansatz, where each Slater determinant is `augmented' with a  symmetric prefactor~\citep{jastrow1955many}, ie $G = p \cdot \mathcal{A}(f_1 \otimes \dots \otimes f_N)$ where $p$ is permutation-invariant. Clearly, $G$ is still antisymmetric, since the product of an antisymmetric function with a symmetric one is again antisymmetric, but grants more representational power. Other parametrisations building from Jastrow are popularly used in the literature, e.g. \emph{backflow} \citep{feynman1956energy}, which models particle interactions by composing the Slater determinant with a permutation equivariant change of variables. 
Among practitioners, it is common knowledge that the Slater Ansatz is inefficient, compared to Jastrow or other more advanced parameterizations.  Yet, there is no proven separation evinced by a particular hard antisymmetric function. We note that the Jastrow ansatz is strictly generalized by backflow (see Section \ref{sec:prelim}), so separations between Slater and Jastrow would have immediate consequences for separations from the stronger architectures as well.

In this work, we are interested in understanding quantitative differences in approximation power between these two classes. Specifically, we wish to find antisymmetric target functions $G$ such that $G$ can be efficiently approximated with the Jastrow ansatz, i.e. approximated to $\epsilon$ error in the infinity norm with some modest dependence on the parameters $N$ and $\epsilon$ by a single Slater determinant with a single symmetric prefactor, yet no Slater representation can approximate $G$ for reasonably small widths.  This question mirrors the issue of depth separation in deep learning theory, where one seeks functions that exhibit a separation between, for example, two layer and three layer networks~\citep{eldan2016power}, as well as recent separations between classes of symmetric representations \citep{zweig2022exponential}.

\paragraph{Main Contribution:}
We prove the first explicit separation between the two ansätze, and construct an antisymmetric function $G$ such that:
\begin{itemize}
    \item In some norm, $G$ cannot be approximated better than constant error by the Slater ansatz, unless there are $\Omega(e^{N^2})$ many Slater Determinants.
    \item $G$ can be written in the Jastrow ansatz with neural network widths bounded by $poly(N)$ using analytic activation functions
\end{itemize}
\section{Related Work}

\subsection{Machine Learning for Quantum Chemistry}

Numerous works explore how to use neural network parameterizations effectively to solve Schr\"odinger's equation.  These include works in first quantization, which try to parameterize the wavefunction $\Psi$ directly~\citep{pfau2020ab,hermann2020deep}, and second quantization, where the wavefunction is restricted to an exponentially large but finite-dimensional Hilbert space, and then the problem is mapped to a spin system~\citep{carleo2017solving}.

\subsection{Antisymmetric Ansätze}

Numerous architectures enforce antisymmetry.  In this work we focus primarily on the Slater ansatz and Jastrow ansatz, but others exist and are used in practice, with associated guarantees of universality~\citep{han2019universal}.  The backflow ansatz enables interaction between particles while preserving antisymmetry~\citep{luo2019backflow}.  More recently, an ansatz that introduces hidden additional fermions was introduced in~\cite{moreno2021fermionic}.

\subsection{Architecture Separations}

A large body of work studies the difference in expressive power between different neural network architectures.  These works frequently center on the representational gap between two-layer and three-layer networks~\citep{eldan2016power,daniely2017depth}.  Relatedly, several works have considered the representational power of different networks architectures constrained to be symmetric~\citep{wagstaff2019limitations,wagstaff2022universal,zweig2022exponential}.

The most closely related work to ours is~\cite{huang2021geometry}, which proves a non-constructive limit on the representability of the backflow ansatz, but requires exact representation rather than approximation in some norm.  Conversely, \cite{hutter2020representing} demonstrates the universality of a single backflow ansatz, but requires a highly discontinuous backflow transform that may not be efficiently representable with a neural network.
\section{Preliminaries and Main Theorem}
\label{sec:prelim}
\subsection{Antisymmetric Ansätze}
We consider $N$ particles restricted to the complex unit circle. That is, $x \in \Omega^N$ with $\Omega=\{ z \in \mathbb{C}; |z|=1\}$.  We denote the tensor product $\otimes$ where, for $f,g: \Omega \rightarrow \mathbb{C}$, we have $f\otimes g : \Omega^{2} \rightarrow \mathbb{C}$ such that $(f\otimes g) (x, y) = f(x) g(y)$.
Given a permutation $\sigma \in \mathcal{S}_N$, and $x \in \Omega^N$, we denote by $\sigma.x  = ( x_{\sigma(1)}, \ldots, x_{\sigma(N)}) \in \Omega^N$ the natural group action. 

Let $\A$ denote the antisymmetric projection, defined via:
\begin{align}
    \A(\phi_1 \otimes \dots \otimes \phi_N) = \frac{1}{N!} \sum_{\sigma \in S_N} (-1)^\sigma \phi_{\sigma(1)} \otimes \dots \otimes \phi_{\sigma(N)}
\end{align}
So up to rescaling we can consider Slater determinants as terms of the form $\A(\phi_1 \otimes \dots \otimes \phi_N)$.  Each $\phi_n$ is called an \emph{orbital}.  Intuitively, a Slater determinant is the simplest way to write an antisymmetric function, inheriting the antisymmetry property from the determinant itself.

Thus the Slater determinant ansatz with $L$ terms can be written as:
\begin{align}
    F = \sum_{l=1}^L \A(f_1^l \otimes \dots \otimes f_N^l)~.
\end{align}
Similarly, the Jastrow ansatz (with only one term) \citep{jastrow1955many} can be written as:
\begin{align}
    G = p \cdot \A(g_1 \otimes \dots \otimes g_N)
\end{align}
where $p$ is a symmetric function, namely $p(\sigma.x) = p(x)$ for any $\sigma$ and $x$. It is immediate to verify that $G$ is antisymmetric. Finally, the Backflow ansatz \citep{feynman1956energy} (considering again a single term) is defined as 
\begin{align}
    \widetilde{G}(x) = \mathcal{A}(\tilde{g}_1 \otimes \dots \tilde{g}_N)(\Phi(x))~,
\end{align}
where $\Phi: \Omega^N \to \widetilde{\Omega}^N$ is an equivariant flow, satisfying $ \Phi(\sigma.x) = \sigma.\Phi(x)$, and where in general $\widetilde{\Omega}$ might be higher-dimensional than $\Omega$. 

In particular, we verify that 
\begin{align}
\Phi : \Omega^N &\to (\mathbb{C} \times \Omega)^N, \\
x&\mapsto \Phi(x) := ( (p^{1/N}(x); x_1), \dots, (p^{1/N}(x); x_N))   
\end{align}
 is equivariant. Given a collection of $N$ orbitals $\phi_1, \ldots \phi_N: \Omega \to \mathbb{C}$, we verify that the Jastrow Ansatz $G = p \cdot \A(g_1 \otimes \dots \otimes g_N)$ can be written as 
$G= \mathcal{A}(\tilde{g}_1 \otimes \dots \tilde{g}_N) \circ \Phi$, with 
$\tilde{g}_j : \mathbb{C} \times \Omega \to \mathbb{C}$ defined as $\tilde{g}_j( z, x)= z \cdot g_j(x)$. 
Thus, the Jastrow ansatz can be recovered as a particular case of the more general backflow ansatz. Therefore, quantitative separations between Slater and Jastrow automatically imply the same rates for Backflow. 

%\joan{Define backflow and show that it recovers Jastrow}

\subsection{Inner Products}

To measure the distance between the Slater Determinant ansatz and the Jastrow ansatz, we need an appropriate norm.

For univariate functions $f, g: \S \rightarrow \mathbb{C}$, define the inner product:
\begin{align}
    \langle f, g \rangle := \frac{1}{(2\pi)} \int_0^{2\pi} f(e^{i\theta}) \conj{g(e^{i\theta})} d\theta ~.
\end{align}
For functions acting on $N$ particles, $F, G: (\S)^N \rightarrow \mathbb{C}$, the associated inner product is
\begin{align}
    \langle F, G \rangle := \frac{1}{(2\pi)^N} \int_{[0,2\pi]^N} F(e^{i\btheta}) \conj{G(e^{i\btheta})} d\btheta ~.
\end{align}

Let us introduce the notation that for $x \in \mathbb{C}^N$ and $\alpha \in \N^N$, $x^\alpha = \prod_{i=1}^N x_i^{\alpha_i}$.  Then the orthogonality of the Fourier basis may be written as $\langle x^\alpha, x^\beta \rangle = \delta_{\alpha \beta}$.

\subsection{Theorem Statement}

With this, we may state our main result explicitly:

\begin{theorem}
    Consider a Slater ansatz with $L$ terms:
\begin{align}
    F = \sum_{l=1}^L \A(f_1^l \otimes \dots \otimes f_N^l)
\end{align}
parameterized by orbitals $f_n^l : \S \rightarrow \mathbb{C}$,
        and a Jastrow ansatz
    \begin{align}
        \hat{G} = p\cdot \A(g_1 \otimes \dots \otimes g_N)
    \end{align}
    parameterized by orbitals $g_n : \S \rightarrow \mathbb{C}$ and a symmetric Jastrow factor $p: (\S)^N \rightarrow \mathbb{C}$.
    
    Let $N\geq 6$ and $1 > \epsilon > 0$.  Then there is a hard function $G$ with $\|G\|=1$, such that $\hat{G}$ parameterized by a neural network with width, depth, and weights scaling in $O(poly(N, 1/\epsilon))$ that can approximate $G$:
    \begin{align}
        \|\hat{G}-G\|_\infty < \epsilon
    \end{align}
    but, for a number of Slater determinants $L \leq e^{N^2}$: 
    \begin{align}
    \min_F \|F - G\|^2 \geq \frac{3}{10}~.
\end{align}

\end{theorem}

Note the constraint $\|G\| = 1$ is just to prevent vacuous solutions by scaling $G$ to be arbitrarily big.  We will describe the exact network structure of the hard function $G$ in Section~\ref{sec:representing-g}.

We also remark that $N^{O(N)} \ll e^{N^2}$, so this is a true separation.  Furthermore, the construction only requires $N^{O(N)}$ parameters due to the analysis of a generic complex ReLU activation.  For particular choices of the activation, the network only requires $poly(N)$ parameters.
Finally, the restriction of $\hat{G}$ to only have one determinant is artificial, here to demonstrate the nature of the separation.  In practice, the Jastrow ansatz allows for multiple determinant terms in learning.

\subsection{Schur Polynomials}

To build up the difficult function $G$, we use several identities related to the symmetric Schur polynomials.  First, we introduce partitions as they will be used to index Schur polynomials:

\begin{figure}
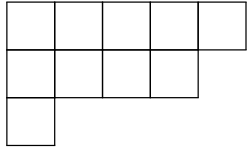
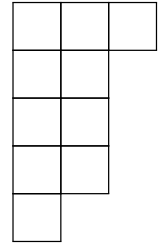

     \centering
     \begin{subfigure}[b]{0.3\textwidth}
         \centering
         \ytableausetup{centertableaux}
         \ydiagram{5,4,1}
         \caption{$\lambda = (5,4,1)$}
         \label{fig:y equals x}
     \end{subfigure}
     \hfill
     \begin{subfigure}[b]{0.3\textwidth}
         \centering
         \ytableausetup{centertableaux}
         \ydiagram{3,2,2,2,1}
         
         \caption{$\lambda' = (3,2,2,2,1)$}
         \label{fig:three sin x}
     \end{subfigure}
     % \hfill
        \caption{Example of Young diagram and conjugate partition}
        \label{fig:young}
\end{figure}

\begin{definition}
    An \emph{integer partition} $\lambda$ is non-increasing, finite sequence of positive integers $\lambda_1 \geq \lambda_2 \geq \dots \geq \lambda_k$.  The weight of the partition is given by $|\lambda| = \sum_{i=1}^k \lambda_i$.  The length of a partition $l(\lambda)$ is the number of terms in the sequence.  We call a partition \emph{even} if every $\lambda_i$ is even.
\end{definition}
Partitions can be represented by their Young diagram, see Figure~\ref{fig:young}.  Furthermore, we will need the notion of a conjugate partition:
\begin{definition}
    Given a partition $\lambda$, the \emph{conjugate partition} $\lambda'$ is gotten by reflecting the Young diagram of $\lambda$ along the line $y=-x$.  We call a partition \emph{doubly even} if $\lambda$ and $\lambda'$ are both even.
\end{definition}

First, we introduce the Vandermonde written as:
\begin{align}
    V(x) = \prod_{i<j} (x_j - x_i)~.
\end{align}
Then we denote the Schur polynomial indexed by partition $\lambda$ as:
\begin{align}
    s_\lambda(x) := \begin{cases} 
      V(x)^{-1} \text{det}\left[x_i^{\lambda_j + N - j}\right] & l(\lambda) \leq N ~,\\
      0 & l(\lambda) > N~.
   \end{cases}
\end{align}
Given two partitions $\lambda$ and $\mu$, the following fact follows easily from linearity of the determinant:
\begin{align}
    \langle s_\lambda \cdot V, s_\mu \cdot V \rangle = N! \cdot \delta_{\lambda \mu}~.
\end{align}
We will in the sequel assume $N$ is even, then we can cite the following formal identity of a particular Pfaffian:
\begin{theorem}[\cite{sundquist1996two} Theorem 5.2, \cite{ishikawa2006generalizations} Corollary 4.2]\label{thm:pfaffian}
\begin{align}
    \sum_{\lambda \text{ doubly even}} s_\lambda \cdot V & = \mathrm{Pf}\left[\frac{x_i - x_j}{1-x_i^2x_j^2} \right]\\
    & = \prod_{i<j} \frac{1}{1-x_i^2x_j^2} \cdot N! \cdot \A(\phi_1 \otimes \dots \otimes \phi_N)~,
\end{align}
where we set the $\phi$ maps to be:
\begin{align}
    \phi_j(x_i) = \begin{cases} 
      x_i(x_i^2)^{N/2-j}(1+x_i^4)^{j-1} & 1 \leq j \leq N/2 \\
      (x_i^2)^{N-j}(1+x_i^4)^{j-1-N/2} & N/2+1 \leq j \leq N
   \end{cases}
\end{align}
\end{theorem}

We note the requirement for $N$ to be even is to ensure the existence of doubly even partitions.  We will leverage this identity to show the existence of a function that requires many Slater determinants to approximate, but can nevertheless be written in the form of the Jastrow ansatz.
\section{Proof Sketch}

We give here a sketch of the tactic of the proof.  Suppose we consider the domain of our antisymmetric functions to be only a finite set of points.  Equivalently, we simplify the problem to a question of antisymmetric tensor products of vectors instead of functions.  Then inapproximability of Slater determinants amounts to finding a high-rank antisymmetric tensor, that cannot be approximated in the $l_2$ norm by a small number of rank-one tensors.

Let's further simplify, and forget about antisymmetry for a moment.  The usual trick for questions of high-rank tensors, is to flatten all the tensors and rewrite them as matrices, so the problem reduces to approximating a high rank matrix by a low rank one.  This is handily solved by SVD.  To make the SVD tractable, the high-rank matrix is typically taken to be diagonal.

In our setting, we cannot simply choose a tensor that will be diagonal after flattening, because the constraints of antisymmetry will enforce certain matrix elements to be equal.  It turns out we can focus on a particular subtensor, where it's possible to flatten to a diagonal matrix, while nevertheless keeping our hard function representable by the Jastrow ansatz.

Indeed, the hard function $G$ can be written exactly in the Jastrow ansatz by the above identity:

\begin{align}
    G & := \frac{C}{\sqrt{N!}} \sum_{\lambda \text{ doubly even}} r^{\left(|\lambda| + \frac{N(N-1)}{2}\right)} s_\lambda \cdot V \\ &= C \sqrt{N!} \cdot  \prod_{i<j} \frac{1}{1-r^4x_i^2x_j^2} \cdot \A(\phi_1^{(r)} \otimes \dots \otimes \phi_N^{(r)})~,
\end{align}

where $\phi_j^{(r)}(z) = \phi_j(rz)$ and $C, r$ are constants with $|r| < 1$.  The second equality gives a Jastrow ansatz.  So it remains to demonstrate why the first expression in terms of Schur polynomials is difficult to approximate, i.e. why it encodes a high-rank antisymmetric tensor.

Let us explain the significance of restricting the support to doubly even Schur functions.  Let $\delta = (N-1, N-2, \dots, 1, 0)$.  Then by simply canceling the Vandermonde factor, for an appropriate partition $\lambda$, we have:
$$
s_\lambda(x) \cdot V(x) =  \text{det}\left[x_i^{\lambda_j + \delta_j}\right]
$$

Furthermore, if $\lambda$ is doubly even, then $\lambda + \delta$ will take the form $(2a + 1, 2a, 2b+1, 2b, \dots)$ with alternating odd and even terms with the odd term one above the subsequent even term.  See Figure~\ref{fig:three graphs} for an example.

The significance of this structure is that, by knowing only the odd values of $\lambda + \delta$, the even values are determined.  We will make essential use of this property to flatten an antisymmetric tensor into a matrix with sets of odd indices as rows, and sets of even indices as columns.  For this matrix, the above structure implies diagonality, and from there we can proceed with a usual proof of low-rank approximation from SVD.

The complete proof is given in Appendix~\ref{sec:proof}.

\begin{figure}
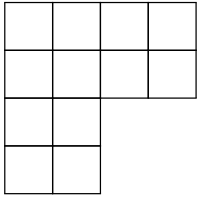
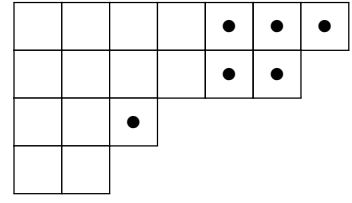

     \centering
     \begin{subfigure}[b]{0.3\textwidth}
         \centering
         \ytableausetup{centertableaux}
         \ydiagram{4,4,2,2}
         \caption{$\lambda = (4,4,2,2)$}
         \label{fig:y equals x}
     \end{subfigure}
     \hfill
     \begin{subfigure}[b]{0.3\textwidth}
         \centering
         \ytableausetup{centertableaux}
         %\ydiagram{7,6,3,2}
         
         \ydiagram[*(white) \bullet]
         {4+3,4+2,2+1,2+0}
         *[*(white)]{4,4,2,2}
         
         \caption{$\lambda + \delta = (7,6,3,2)$}
         \label{fig:three sin x}
     \end{subfigure}
     \hfill
        \caption{$\lambda$ and $\lambda + \delta$ for $\lambda$ doubly even. \\ Note that $\lambda + \delta \sim (6+1,2+1) \cup (6,2).$}
        \label{fig:three graphs}
\end{figure}
\section{Proof}\label{sec:proof}

\subsection{Building the hard function $G$}

For any $r \in \mathbb{R}$ with $|r| < 1$, by mapping $x \mapsto rx$ and using homogeneity of $s_\lambda$ and $V$, we define $G$ via the generating function identity:

\begin{align}
    G := \frac{C}{\sqrt{N!}} \sum_{\lambda \text{ doubly even}} r^{\left(|\lambda| + \frac{N(N-1)}{2}\right)} s_\lambda \cdot V &= C \sqrt{N!} \cdot  \prod_{i<j} \frac{1}{1-r^4x_i^2x_j^2} \cdot \A(\phi_1^{(r)} \otimes \dots \otimes \phi_N^{(r)})~,
\end{align}
where
\begin{align}
    \phi_j^{(r)}(x_i) = \begin{cases} 
      rx_i((rx_i)^2)^{N/2-j}(1+(rx_i)^4)^{j-1} & \text{ if } 1 \leq j \leq N/2 \\
      ((rx_i)^2)^{N-j}(1+(rx_i)^4)^{j-1-N/2} & \text{ if } N/2+1 \leq j \leq N
   \end{cases}
\end{align}
where $C$ is chosen to normalize $G$.

Note that from the RHS, it is clear that $G$ is written in the form of a Jastrow ansatz.  We will discuss efficiency of computing $G$ further below.

It remains to choose $r$ and $C$ such that $\|G\| = 1$.  Note that,  if $p(k)$ denotes the number of partitions of $k$, and and $p'(k)$ denotes the number of doubly even partitions of $k$ , it's easy to see that
\begin{align}
    p'(k) = \begin{cases} 
      p(k/4) & k \equiv 0 \mod 4 \\
      0 & else
   \end{cases}
\end{align}
So we calculate by orthogonality:
\begin{align}
    \|G\|^2 & = \frac{C^2}{N!} \left\langle \sum_{\lambda \text{ doubly even}} r^{\left(|\lambda| + \frac{N(N-1)}{2}\right)} s_\lambda \cdot V, \sum_{\mu \text{ doubly even}} r^{\left(|\mu| + \frac{N(N-1)}{2}\right)} s_\mu \cdot V \right\rangle \\
    &= C^2 r^{N(N-1)} \sum_{\lambda \text{ doubly even}} r^{2|\lambda|} \\
    &= C^2 r^{N(N-1)} \sum_{k=0}^\infty r^{2k} p'(k) \\
    &= C^2 r^{N(N-1)} \sum_{k=0}^\infty r^{8k} p(k) \\
    &= C^2 r^{N(N-1)} \prod_{k=1}^\infty \frac{1}{1-r^{8k}}
\end{align}
where in the last line we employ the generating function for partition numbers.
Then setting $C = \left(r^{-N(N-1)}\prod_{k=1}^\infty 1-r^{8k}\right)^{1/2}$ gives $\|G\| = 1$.

\subsection{From Tensors to Matrices}

The point of choosing $G$ in this way, is it enables a simple flattening argument, where we can reduce comparing tensors to comparing matrices.

Note again that terms of the form $x^\alpha$ for $\alpha \in \N^N$ are orthonormal.  Hence, we derive an initial lower bound by Bessel's inequality:
\begin{align}
    \|F - G\|^2 & \geq \sum_{\alpha \in \N^N} \left(\left\langle F, x^\alpha \right\rangle - \left\langle G, x^\alpha \right\rangle\right)^2~.
\end{align}
Note that by antisymmetry of $F$ and $G$, if $\alpha$ doesn't have distinct elements then
\begin{align}
    \left\langle F, x^\alpha \right\rangle = \left\langle G, x^\alpha \right\rangle = 0~.
\end{align}
To see this, suppose $\alpha_1 = \alpha_2$, and let $P_{12}$ be the permutation operator defined by 
\begin{align}
P_{12}F(x_1, x_2, x_3, \dots) = F(x_2, x_1, x_3, \dots)
\end{align}
It's easy to see $P_{12}$ is a symmetric operator with respect to $\langle \cdot, \cdot \rangle$.  Then for any antisymmetric function $H$, 
\begin{align}
    \langle H, x^\alpha \rangle & = \langle H, P_{12}x^\alpha \rangle\\
    & = \langle P_{12} H, x^\alpha \rangle \\
    & = -\langle H, x^\alpha \rangle
\end{align}
which implies $\langle H, x^\alpha \rangle = 0$.

Furthermore, let us define the equivalence class $\sim$ as via $\alpha \sim \alpha'$ if there exists a permutation $\pi$ such that $\alpha = \pi \circ \alpha'$.  Then by similar reasoning, $\alpha \sim \alpha'$ implies:
\begin{align}
    \left\langle F, x^\alpha \right\rangle & = (-1)^\pi \left\langle F, x^{\alpha'} \right\rangle \\
    \left\langle G, x^\alpha \right\rangle & = (-1)^\pi \left\langle G, x^{\alpha'} \right\rangle     
\end{align}
So define $\N_{\geq}^N$ to be the set of strictly decreasing non-negative integer vectors of length $N$, then we have:
\begin{align}
    \|F - G\|^2 
    & \geq N! \cdot \sum_{\alpha \in \N_\geq^N} \left(\left\langle F, x^\alpha \right\rangle - \left\langle G, x^\alpha \right\rangle\right)^2
\end{align}
Now, we can consider a flattening argument, by passing from tensors to matrices.  Define 
\begin{align}
    \mathfrak{A}_1 &= \{\beta \in \N_\geq^{N/2}: \beta_i \equiv 1 \mod 2\}\\
    \mathfrak{A}_2 &= \{\gamma \in \N_\geq^{N/2}: \gamma_i \equiv 0 \mod 2\}
\end{align} 

For $\beta \in \mathfrak{A}_1$ and $\gamma \in \mathfrak{A}_2$, let $\beta \cup \gamma \in \N^N$ be the concatenation of $\beta$ and $\gamma$.

Then given a function acting on $N$ particles such as $G$, we can map $G$ to a (infinite-dimensional) matrix $M$ indexed by the sets $\mathfrak{A}_1$ and $\mathfrak{A}_2$:
\begin{align}
    M(G) = \left[\langle G, x^{\beta \cup \gamma}\rangle \right]_{\beta, \gamma}
\end{align}
Let us calculate the entries of this matrix.  Let $\delta = (N-1, N-2,  \dots, 1,  0)$, and observe that:
\begin{align}
    \langle s_\lambda \cdot V, x^{\beta \cup \gamma} \rangle = \begin{cases} 
      \pm 1 & \lambda + \delta \sim \beta \cup \gamma~, \\
      0 & \text{otherwise.}
   \end{cases}
\end{align}

Note that ambiguity in sign depends on the sign of the permutation that maps $\lambda + \delta$ to $\beta \cup \gamma$.

By definition, $G$ is a sum of terms of the form $s_\lambda \cdot V$ where $\lambda$ is doubly even.  This implies that $\lambda + \delta = (2a_1+1, 2a_1, 2a_2+1, 2a_2, \dots)$ with $a_i > a_{i+1}$.  In other words, $\lambda + \delta \sim (\gamma + \one) \cup \gamma$ with $\gamma + \one \in \mathfrak{A}_1$ and $\gamma \in \mathfrak{A}_2$, where $\one$ is the all-ones vector.  See Figure~\ref{fig:three graphs} for an example.

It follows that we may write:
\begin{align}
    \langle G, x^{\beta \cup \gamma} \rangle = \begin{cases} 
      \pm \frac{C}{\sqrt{N!}}r^{\left(|\lambda| + \frac{N(N-1)}{2}\right)} & \beta = (\gamma + \one), \quad \lambda + \delta \sim (\gamma + \one) \cup \gamma~, \\
      0 & \text{otherwise.}
   \end{cases}
\end{align}

Suppose we index $M(G)$ such that the $i$th column is indexed by $\gamma^{(i)}$, and the $i$th row is indexed by $\gamma^{(i)} + \one$.  Then $M(G)$ is in fact a diagonal matrix.  And given the functional form of $G$, we have that the diagonal terms will include:
\begin{itemize}
    \item $\pm\frac{C}{\sqrt{N!}} r^{\left(0 + \frac{N(N-1)}{2}\right)}$ repeated $p(0)$ times,
    \item $\pm\frac{C}{\sqrt{N!}} r^{\left(4 + \frac{N(N-1)}{2}\right)}$ repeated $p(1)$ times,
    \item $\pm\frac{C}{\sqrt{N!}} r^{\left(8 + \frac{N(N-1)}{2}\right)}$ repeated $p(2)$ times,
    \item \dots
    \item $\pm\frac{C}{\sqrt{N!}} r^{\left(4k + \frac{N(N-1)}{2}\right)}$ repeated $p(k)$ times.

\end{itemize}

Second, let us consider $M(f_1 \otimes \dots \otimes f_N)$.  We can calculate the inner product of a rank-one function as the product of orbital inner products:
\begin{align}
    \langle f_1 \otimes \dots \otimes f_N, x^{\beta \cup \gamma} \rangle &= \prod_{n=1}^{N/2} \langle f_n, y^{\beta_n}\rangle \prod_{n=1}^{N/2} \langle f_{N/2 + n}, y^{\gamma_n} \rangle
\end{align}

where we introduce $y \in \mathbb{C}$ as a one-dimensional dummy variable for integration.

Define vectors $u \in \mathbb{C}^{|\mathfrak{A}_1|}$ and $v \in \mathbb{C}^{|\mathfrak{A}_2|}$ such that 
\begin{align}
    u_\beta &= \prod_{n=1}^{N/2} \langle f_n, y^{\beta_n}\rangle ~,\\
    v_\gamma &= \prod_{n=1}^{N/2} \langle f_{N/2 + n}, y^{\gamma_n}\rangle~.
\end{align}
Then it's clear that $M(f_1 \otimes \dots \otimes f_N) = uv^T$, i.e. it is rank-one.
Consequently, because $F$ is the sum of $L \cdot N!$ rank-one tensors, $M(F)$ will be rank at most $L \cdot N!$.

So we finally pass from tensors to matrices, and lower bound via the Frobenius norm $\|\cdot\|_F$:
\begin{align}
    \|F - G\|^2 
    & \geq N! \cdot \sum_{\alpha \in \N_\geq^N} \left(\left\langle F, x^\alpha \right\rangle - \left\langle G, x^\alpha \right\rangle\right)^2 \\
    & \geq N! \cdot \|M(F) - M(G)\|_F^2
\end{align}

Thus, because $M(F)$ is low-rank and $M(G)$ is chosen to be diagonal, we have an approachable infinite-dimensional matrix low-rank approximation problem.

\subsection{Deriving the bound}
By SVD, the optimal choice for $F$ is to produce a diagonal matrix $M(F)$ of rank $L \cdot N!$ with the maximal singular values of $G$ along the diagonal.
So it only remains to calculate these terms, and lower bound the approximation.

So suppose we choose $L \leq e^{N^2}$.  Noting that $N^N \leq e^{N^2}/14$ for $N\geq 6$:
\begin{align}
    L \cdot N! & \leq e^{N^2} N^N \\
    & \leq e^{2N^2} / 14 \\
    & \leq p(N^4)
\end{align}

where the last line follows from Corollary 3.1 in~\cite{maroti2003elementary}.

So clearly $L \leq e^{N^2}$ guarantees that $L \cdot N! \leq \sum_{k=0}^{N^4} p(k)$.

Thus, since $M(F)$ is constrained to have rank $\leq L \cdot N!$, it will be diagonal with $\leq \sum_{k=0}^{N^4} p(k)$ terms, so that:
\begin{align}
    \|F - G\|^2 & \geq N!\cdot\|M(F) - M(G)\|_F^2 \\
    & \geq N! \cdot \sum_{k=N^4 + 1}^\infty \left(\pm \frac{C}{\sqrt{N!}} r^{\left(4k + \frac{N(N-1)}{2}\right)} \right)^2 p(k) \\
    & = C^2 r^{N(N-1)}\sum_{k=N^4 + 1}^\infty r^{8k} p(k) \\
    & = 1 - C^2 r^{N(N-1)} \sum_{k=0}^{N^4} r^{8k} p(k)
\end{align}
where the last line follows as $C$ was chosen so that $C^2 r^{N(N-1)} \sum_{k=0}^{\infty} r^{8k} p(k) = 1$.

Note that 
\begin{align}
    \sum_{k=0}^{N^4} r^{8k} p(k) \leq \prod_{k=1}^{N^4} \frac{1}{1-r^{8k}}
\end{align}

as the LHS is the generating function for partitions $\lambda$ with $|\lambda| \leq N^4$, and the RHS is the generating function for partitions with all parts $\leq N^4$, which clearly dominates the LHS termwise.

So plugging back in the definition of $C = \left(r^{-N(N-1)}\prod_{k=1}^\infty 1-r^{8k}\right)^{1/2}$:
\begin{align}
    \|F - G\|^2 & \geq 1 - C^2 r^{N(N-1)} \prod_{k=1}^{N^4} \frac{1}{1-r^{8k}} \\
    & = 1 - \prod_{k=N^4+1}^\infty 1-r^{8k}~.
\end{align}
Finally, by choosing $r = 1 - \frac{1}{8N^4 + 8}$, we have:
\begin{align}
    1 - \prod_{k=N^4+1}^\infty 1-r^{8k} & \geq 1 - \left(1 - r^{8N^4 + 8}\right) \\
    & = \left(1 - \frac{1}{8N^4 + 8}\right)^{8N^4 + 8} \\
    & \geq \left(1 - \frac{1}{16}\right)^{16} \\
    & \geq \frac{3}{10}~,
\end{align}
where we use that the limit $\left(1 - \frac{1}{n}\right)^n$ increases monotonically in $n$.  Hence, we conclude:
\begin{align}
    \|F - G\|^2 & \geq \frac{3}{10}~.
\end{align}

\subsection{Efficiency of representing $G$}\label{sec:representing-g}

We remind the representation, where with $r = 1 - \frac{1}{8N^4 + 8}$ we have:
\begin{align}
    G = C \sqrt{N!} \cdot  \prod_{i<j} \frac{1}{1-r^4x_i^2x_j^2} \cdot \A(\phi_1^{(r)} \otimes \dots \otimes \phi_N^{(r)})
\end{align}
with
\begin{align}
    \phi_j^{(r)}(x_i) = \begin{cases} 
      rx_i((rx_i)^2)^{N/2-j}(1+(rx_i)^4)^{j-1} & 1 \leq j \leq N/2 ~,\\
      ((rx_i)^2)^{N-j}(1+(rx_i)^4)^{j-1-N/2} & N/2+1 \leq j \leq N~.
   \end{cases}
\end{align}

So it remains to characterize some $\hat{G}$ in the Jastrow ansatz, parameterized with neural networks, that approximates $G$.

%%%%%%%%%%%%%%%

Given inputs of the form $z \in \mathbb{C}^n$, we consider network layers of the form $z \mapsto \sigma(Az + B\conj{z} + c)$ where $A, B, c$ are learnable weights, and $\sigma$ is an analytic activation function with coefficients that decay at least factorially (such as $\exp$, $\sinh$ or $\sin$).  It is easy to see by the complex Stone-Weierstrauss theorem that linear combinations of these terms, i.e. two layer networks, are universal approximators for continuous complex-valued functions.

Observe further, the following fact:

\begin{lemma}\label{lem:unity}
    Fix $J$ and let $\gamma$ be a primitive $J$th root of unity.  Then
    \begin{align}
    \frac{1}{J}\sum_{j=0}^{J-1} \gamma^{ij} = 
        \begin{cases} 
      1 & i \equiv 0 \mod J \\
      0 & i \not\equiv 0 \mod J
  \end{cases}
    \end{align}
\end{lemma}

We will use this fact to exponentially well-approximate all the components of the hard function $G$.  

\begin{lemma}\label{lem:expnet}
     For any $k, J \in \N$ with $2ek < J$, there exists a shallow neural network $f^{(k)}$ using either the $\exp, \sinh, \sin$ activations,  with $O(J)$ neurons and O(k) weights, such that
    \begin{align}
        \sup_{|\xi| \leq 2}\left|f^{(k)}(\xi) - \xi^k\right| &\leq 2 \left(\frac{2ek}{J}\right)^J\\
    \end{align}
\end{lemma}

\begin{proof}

    The result is obvious  if $k = 0$, so assume $k \geq 1$.

    Let $\gamma$ be a primitive $J$th root of unity, $t \in \mathbb{C}$ some constant, and let $k \in \N$ such that $0 \leq k \leq J-1$.  Let $\sigma$ denote a chosen activation, with Taylor expansion $\sigma(z) = \sum_{i=0}^\infty c_i z^i$.
    
    By applying Lemma~\ref{lem:unity} we can define a network $f^{(k)}$ and expand as:
    \begin{align}
        f^{(k)}(\xi) & := \sum_{j=0}^{J-1} \frac{\gamma^{-kj}}{J} \sigma(\gamma^j t \xi) \\
        & = \sum_{j=0}^{J-1} \frac{\gamma^{-kj}}{J} \sum_{i=0}^\infty c_i(\gamma^j t \xi)^i\\
        & = \sum_{i=0}^\infty c_i (t\xi)^i \left[\frac{1}{J} \sum_{j=0}^{J-1} \gamma^{(i-k)j}\right] \\
        & = \sum_{i=0}^\infty c_i (t\xi)^i \ind_{i \equiv k \mod J} \\
        & = \sum_{i=0}^\infty c_{iJ + k} (t\xi)^{iJ+k} \\
        & = c_{k} (t\xi)^{k} + \sum_{i=1}^\infty c_{iJ + k} (t\xi)^{iJ+k}
    \end{align}

    Let's choose $t = c_k^{-1/k}$. 
 Therefore, it follows:

    \begin{align}
        \sup_{|\xi| \leq 1} |f^{(k)}(\xi) - \xi^k| \leq \sum_{i=1}^\infty |c_{iJ + k}| |(2t)^{iJ + k}|
    \end{align}

    If $\sigma(\xi) = \exp(\xi)$, it follows $c_k = \frac{1}{k!}$.  So by Stirling's approximation we have $|t| \leq k$ and therefore if $J > 2ek$, then:

    \begin{align}
        \sum_{i=1}^\infty |c_{iJ + k}| |(2t)^{iJ + k}| & \leq \sum_{i=1}^\infty \frac{(2k)^{iJ + k}}{(iJ+k)!} \\
        & \leq \sum_{i=1}^\infty \left(\frac{2ek}{J}\right)^{iJ}\\
        & \leq 2 \left(\frac{2ek}{J}\right)^J
    \end{align}

    We cannot apply this argument directly for the $\sin$ activation, because it is an odd function, and $c_k = 0$ if $k$ is even.  However, a subset of networks with the $\sin$ activations is networks with the activation $\sigma(\xi) = \sin \xi + \sin(\pi/2 - \xi) = \sin \xi + \cos \xi$, for which the above argument applies as before.
    
    Similarly, using the identity $\sinh(\xi + 1) = \sinh \xi \cosh 1 + \cosh \xi \sinh 1$, the same bound applies with the activation $\sinh$.

\end{proof}

%%%%%%%%%%%%%%%

% First, we make mention that, for very specific activations, this function $G$ may be written exactly in the Jastrow ansatz.  

% After fixing the value of $r$, we can consider each $\phi_i := \phi_i^{(r)}$ as an activation function acting on a one-dimensional input.  Consider also the ``activation'' $\psi(x_i, x_j) = \frac{1}{1 - r^4 x_i^2 x_j^2}$.  Then $G$ can be clearly written exactly in the Jastrow ansatz (where the Jastrow term $J$ is given as a symmetric network with product pooling) as

% \begin{align}
%     G & = C \sqrt{N!} \cdot  J(x) \cdot \A(\phi_1 \otimes \dots \otimes \phi_N) \\
%     & = C \sqrt{N!} \cdot  \prod_{i<j} \psi(x_i, x_j) \cdot \A(\phi_1 \otimes \dots \otimes \phi_N)
% \end{align}

% We know consider approximation error under a more typical choice of activation function.  We will use the modReLU~\cite{arjovsky2016unitary}:

% \begin{align}
%     \sigma(z) = \begin{cases} 
%       0 & |z| \leq 1\\
%       z- \frac{z}{|z|} & |z| \geq 1
%    \end{cases}
% \end{align}

We consider first the Jastrow factor of $G$.  We will approximate it in $\hat{G}$ using a Relational Network~\citep{santoro2017simple} with multiplication pooling.  In what follows, we consider the infinity norm restricted to the unit complex circle.

Define $y_{ij} = f^{(2)}(x_i + x_j) - f^{(2)}(x_i) - f^{(2)}(x_j)$.  Note that because $x_ix_j = (x_i+x_j)^2 - x_i^2 - x_j^2$, by Lemma~\ref{lem:expnet}, we have for an appropriate choice of $J$:

\begin{align}
\|y_{ij} - x_ix_j\|_\infty & \leq \|f^{(2)}(x_i + x_j) - (x_i + x_j)^2\|_\infty + \|f^{(2)}(x_i) - x_i^2\|_\infty + \|f^{(2)}(x_j) - x_j^2\|_\infty \\
& \leq 6 \left(\frac{4e}{J}\right)^J
\end{align}

Clearly it follows that for sufficiently large $J$, $\|y_{ij}\|_\infty \leq 2$.  Therefore, applying Lemma~\ref{lem:expnet} again we have

\begin{align}
    \left\|f^{(2k)}(y_{ij}) - x_i^{2k}x_j^{2k}\right\|_\infty
    & \leq \left\|f^{(2k)}(y_{ij}) - y_{ij}^{2k}\right\|_\infty + \left\|y_{ij}^{2k} - x_i^{2k}x_j^{2k}\right\|_\infty \\
    & \leq 2 \left(\frac{2ek}{J}\right)^J + \sum_{l=0}^{2k-1} \| y_{ij}^{2k-l}(x_ix_j)^l - y_{ij}^{2k-l-1}(x_ix_j)^{l+1}\|_\infty \\
    & \leq 2 \left(\frac{2ek}{J}\right)^J + \sum_{l=0}^{2k-1} \|y_{ij}^{2k-l-1}\|_\infty \|(x_ix_j)^l \|_\infty \|y_{ij} - x_ix_j\|_\infty \\
    & \leq 2 \left(\frac{2ek}{J}\right)^J + 2k \cdot 2^{2k} \cdot 6 \left(\frac{4k}{J}\right)^J
\end{align}

Now, consider a network $g$ that takes in two inputs, defined via

\begin{align}
    g(x_i, x_j) = 1 + \sum_{k=1}^{K} r^{4k} f^{(2k)}\left(f^{(2)}(x_i + x_j) - f^{(2)}(x_i) - f^{(2)}(x_j) \right)
\end{align}

Then it follows that:

\begin{align}
    \left\|\frac{1}{1 - r^4 x_i^2 x_j^2} - g(x_i,x_j)\right\|_\infty
    & \leq \sum_{k=1}^K r^{4k} O\left(k2^{2k} \left(\frac{2ek}{J}\right)^J \right) + \left\|\sum_{k=K+1}^\infty (r^4x_i^2x_j^2)^k\right\|_\infty \\
    & \leq O\left(K^2 2^{2K} \left(\frac{2eK}{J}\right)^J\right) + O\left (\frac{r^{4K}}{1-r}\right) =: \delta_1
\end{align}

Let us assume we choose $J, K$ such that $\delta_1 \leq 1$.  One can confirm that $\max\left(\left\|\frac{1}{1 - r^4 x_i^2 x_j^2}\right\|_\infty, \|g(x_i,x_j)\|_\infty \right) \leq \frac{1}{1-r} + \delta_1 \leq \frac{2}{1-r}$.  So it follows from routine calculation that
\begin{align}
    \left\|\prod_{i<j} \frac{1}{1 - r^4 x_i^2 x_j^2} - \prod_{i<j} g(x_i,x_j)\right\|_\infty \leq N \left(\frac{2}{1-r}\right)^{N^2 - 1}\delta_1
\end{align}

Consider second the antisymmetric factor.  Following the row transforms given in the proof of Lemma 3.4 in~\cite{ishikawa2006generalizations}, the antisymmetric term may be equivalently written as:
\begin{align}
    \A(\phi_1^{(r)} \otimes \dots \otimes \phi_N^{(r)}) = \A(\psi_1^{(r)} \otimes \dots \otimes \psi_N^{(r)})
\end{align}

with
\begin{align}
    \psi_j^{(r)}(x_i) = \begin{cases} 
      rx_i((rx_i)^2)^{N/2-j}(1+(rx_i)^{4(j-1)}) & 1 \leq j \leq N/2 ~,\\
      ((rx_i)^2)^{N-j}(1+(rx_i)^{4(j-1-N/2)}) & N/2+1 \leq j \leq N~.
   \end{cases}
\end{align}

Expanding, we see that $\psi_j^{(r)}$ is a polynomial in degree $\leq 3N$ with only two non-zero coefficients, each bounded by $1$.  So by Lemma~\ref{lem:expnet} it is easy to construct networks $\hat{\psi}_j$ such that

\begin{align}
    \|\psi_j^{(r)} - \hat\psi_j\|_\infty \leq O\left(\left(\frac{6eN}{J}\right)^J \right) =: \delta_2
\end{align}.

Choose $J$ to ensure $\delta_2 \leq 1$, then we also clearly have that $\max\left(\|\hat\psi_j^{(r)}\|_\infty, \|\hat\psi_j\|_\infty \right) \leq 2 + \delta_2 \leq 3$.  Now, we calculate:
{\small 
\begin{align}
    \left\|\A(\psi_1^{(r)} \otimes \dots \otimes \psi_N^{(r)}) - \A(\hat\psi_1 \otimes \dots \otimes \hat\psi_N)\right\|_\infty & =  \left\|\frac{1}{N!} \sum_{\sigma}  (-1)^\sigma \left(\psi_{\sigma(1)}^{(r)} \otimes \dots \otimes \psi_{\sigma(N)}^{(r)} - \hat\psi_{\sigma(1)} \otimes \dots \otimes \hat\psi_{\sigma(N)}\right) \right\|_\infty \\
    & \leq \frac{1}{N!} \sum_{\sigma} \left\|  \left(\psi_{\sigma(1)}^{(r)} \otimes \dots \otimes \psi_{\sigma(N)}^{(r)} - \hat\psi_{\sigma(1)} \otimes \dots \otimes \hat\psi_{\sigma(N)}\right) \right\|_\infty \\
    & < N3^{N-1} \delta_2
\end{align}}

Finally, we combine the Jastrow factor and antisymmetric component.  Let
\begin{align}
    \hat G(x) = C \sqrt{N!} \prod_{i<j} g(x_i,x_j) \A(\hat\psi_1 \otimes \dots \otimes \hat\psi_N)(x)~.
\end{align}

Again, we have the simple bounds $\left\|\A(\psi_1^{(r)} \otimes \dots \otimes \psi_N^{(r)})\right\|_\infty \leq 2^N$ and $\left\|\prod_{i<j}\frac{1}{1 - r^4 x_i^2 x_j^2}\right\|_\infty \leq \left(\frac{2}{1-r}\right)^{N^2}$.

Then we calculate:
\begin{align}
    \|G - \hat{G}\|_\infty & = C \sqrt{N!} \left\|\prod_{i<j} \frac{1}{1-r^4x_i^2x_j^2} \cdot \A(\phi_1^{(r)} \otimes \dots \otimes \phi_N^{(r)}) - \prod_{i<j} g(x_i, x_j) \cdot \A(\hat\psi_1 \otimes \dots \otimes \hat\psi_N)\right\|_\infty \\
    & \leq C \sqrt{N!} \left(N3^{N-1} \delta_2 \left(\frac{2}{1-r}\right)^{N^2} + N \left(\frac{2}{1-r}\right)^{N^2 - 1}\delta_1 2^N \right) \\
    & \leq \sqrt{N!} N 3^N \left(\frac{2}{1-r}\right)^{N^2} (\delta_1 + \delta_2)
\end{align}

From the choice $r = 1 - \frac{1}{8N^4 + 8}$ and the assumption that $N \geq 6$, we can further bound

\begin{align}
    \|G - \hat{G}\|_\infty
    & \leq N^{2N} (9N^4)^{N^2} (\delta_1 + \delta_2) \\
    & \leq N^{5N^2}(\delta_1 + \delta_2)
\end{align}

Choosing $J \geq 12eK$, and $K \geq 2$ so that $K^2 \leq 2^K$, we recall that

\begin{align}
\delta_1 + \delta_2 &\leq O\left(K^2 2^{2K} \left(\frac{6eK}{J}\right)^J\right) + O\left (\frac{r^{4K}}{1-r}\right) \\
& \leq O\left(2^{3K-J}\right) + O\left (N^4\left(1 - \frac{1}{N^4} \right)^{4K}\right) \\
& \leq O\left(2^{-9K}\right) + O\left (N^4e^{-4K/N^4}\right)\\
& \leq O\left (N^4e^{-4K/N^4}\right)
\end{align}

because the right-most term dominates in the second-to-last line.

Hence, if we'd like to control the error $\|G - \hat{G}\|_\infty$ by some $\epsilon$, we require that for some universal constant $C$, 

\begin{align}
    \epsilon \geq C N^{5N^2 + 4} e^{-4K/N^4}
\end{align}

Note that for $N \geq 6$, we have $N^{5N^2 + 4} \leq e^{N^3}$, and therefore it suffices to choose $K$ such that

\begin{align}
    \epsilon \geq C e^{N^3 - 4K/N^4}
\end{align}

And this condition is equivalent to the bound

\begin{align}
    K \geq \frac{1}{4} \left(N^4 \log \frac{C}{\epsilon} + N^7 \right)
\end{align}

Note that $J$ is subject to the same bound up to constant factors.
\section{Experiments}

We illustrate the nature of this exponential separation in finite case, specifically by seeking to learn our hard function in the Slater ansatz and Jastrow ansatz with $N=4,6$ particles.  In particular, we try to learn the hard function $G$ with the two given ansätze (rescaled with the $C$ term as this constant is extremely small to give normalization in the $L_2$ norm).

\begin{figure*}[h]
\centering
% \begin{subfigure}{0.5\textwidth}
%   \centering
%   \includegraphics[width=1.0\linewidth]{plots/plot_2.png}
%   \caption{$N = 2$}
% \end{subfigure}%

\begin{subfigure}{0.5\textwidth}
  \centering
  \includegraphics[width=1.\linewidth]{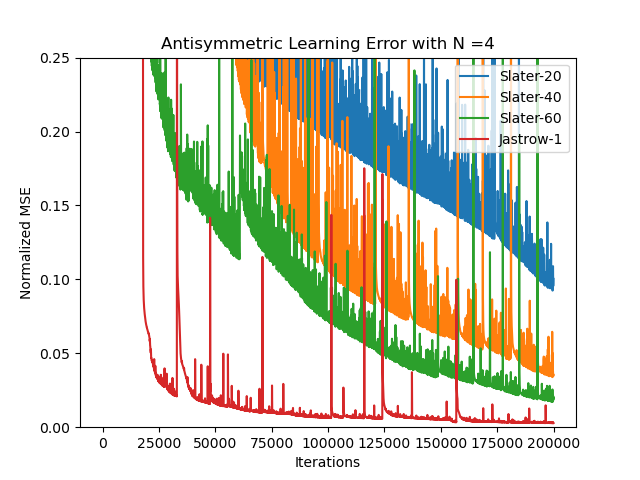}
  \caption{All training iterations with $N=4$}
\end{subfigure}%
\begin{subfigure}{0.5\textwidth}
  \centering
  \includegraphics[width=1.\linewidth]{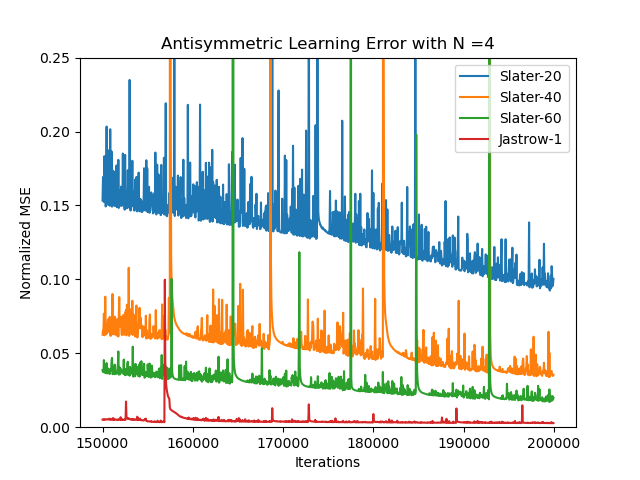}
  \caption{Last quarter of iterations with $N=4$}
\end{subfigure}%

\begin{subfigure}{0.5\textwidth}
  \centering
  \includegraphics[width=1.\linewidth]{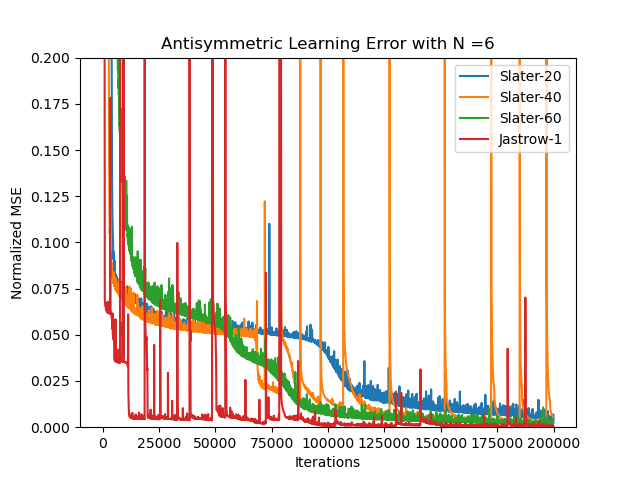}
  \caption{All training iterations with $N=6$}
\end{subfigure}%
\begin{subfigure}{0.5\textwidth}
  \centering
  \includegraphics[width=1.\linewidth]{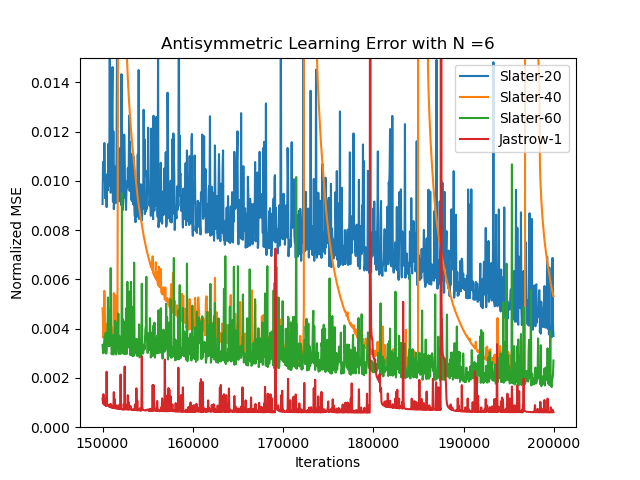}
  \caption{Last quarter of iterations with $N=6$}
\end{subfigure}%

\caption{Training MSE for Slater ansatz of varying number of determinants and Jastrow ansatz of one determinant.}
% \vspace{-0.5cm}
\label{fig:plot}
\end{figure*}

We consider the empirical approximation of mean squared training error in $\|\cdot\|$, where we compare learning with the Slater ansatz and a large number of determinants, vs. the Jastrow ansatz with a single determinant.  We parameterize each orbital with a three layer neural network with hidden width 30, and activation the complex ReLU that maps $a + bi \mapsto ReLU(a) + ReLU(b)i$. We further parameterize the Jastrow term with the symmetric architecture of a Relational Network~\citep{santoro2017simple} and multiplication pooling where all networks are three layers and all hidden widths are also 30.  This is justified by the structure of the Jastrow term in the definition of $G$, which is a symmetric product of terms depending on pairs of particles. 

Learning rate is set to 0.0005 in all runs, for 200000 iterations of full batch gradient descent on 10000 samples drawn i.i.d. from the complex unit circle.  We plot the normalized MSE, i.e. MSE divided by the error attained by the naive zero function.

The results are given in Figure~\ref{fig:plot}.  Each Slater ansatz is labeled by the number of determinants it is parameterized with.  We observe that one Jastrow determinant suffices where a large number of Slater determinants fails to achieve low MSE, alluding to the exponential nature of the separation as the number of particles $N$ increases.

\section{Discussion}

\subsection{Proof limitations}

The proof technique relies on finding a symmetric function that is supported exclusively on doubly even Schur polynomials.  This is established in the Pfaffian identity given in Theorem~\ref{thm:pfaffian}.  This yields a function that requires exponentially many Slater determinants but may be written exactly in Jastrow form.

However, the large magnitude of the Jastrow factor precludes efficient approximation in the infinity norm.  This cannot be overcome by changing the value of $r$: as $r$ approaches 1, the magnitude of support on high dimensional doubly even Schur polynomials increases while simultaneously the magnitude of the Jastrow factor increases.  So we must choose a sufficiently large $r$ in order to guarantee the induced matrices are effectively high-rank.

An alternative tactic would be to control approximation in the $L_2$ norm given by $\|\cdot\|$.  However, calculating the $L_2$ norm is most easily done after decomposing into the orthogonal basis of multinomials, which is challenging when multiplying terms in the Jastrow product.

This proof also only uses 1-dimensional particles to evince a separation.  A nearly identical proof could be employed for higher-dimensional particles by only utilizing the first component, but there would be no dependence on the dimension $d$.  Understanding the simultaneous dependence on $N$ and $d$ would therefore require a new proof technique.

\subsection{Open Questions}

The main result of this work represents a first step in understanding separations between relevant antisymmetric ansätze.  We conclude with a discussion of the natural open questions in this domain.

\paragraph{Practical Wavefunctions}  Extending the analysis to consider wavefunctions that appear in more practical applications would be informative.  For example, in the one-dimensional case, famously the eigenfunctions of the Sutherland model are known~\citep{langmann2005method}, and the representability of these functions in particular is an open question.

\paragraph{More Powerful Ansätze} The separation we demonstrate in this work is between the two simplest ansätze.  Demonstrating separations among the more expressive network architectures, for example the backflow or hidden fermion models discussed previously, would prove the merit of these more complicated and time-intensive methods.

\paragraph{Learnability separations} Our current separation concerns exclusively the approximation properties of the two parametric families of antisymmetric functions, and as such neglects any optimization question. It would be interesting to integrate the optimization aspect in the separation, similarly as in \citep{safran2022optimization} for fully-connected networks.

\paragraph{Acknowledgements:} We are thankful of discussions with Giuseppe Carleo and Michael Lindsey. This work is partially funded by NSF CAREER CIF-1845360,  NSF CCF-1814524 and NSF DMS-2134216.

\bibliography{refs}
\bibliographystyle{iclr2023_conference}

\end{document}